\documentclass[letterpaper]{article}
 
\renewcommand{\cite}[1]{\citet{#1}}

\usepackage{amsmath, amsthm, amssymb, epsfig, sectsty, graphicx, float, url}
\usepackage{bm}

\usepackage{xcolor}
\usepackage{natbib}
\usepackage{hyperref}
\usepackage[margin=1.1in]{geometry}

\usepackage[utf8]{inputenc} 
\usepackage[T1]{fontenc}    
\usepackage{hyperref}       
\usepackage{url}            
\usepackage{booktabs}       
\usepackage{amsfonts}       
\usepackage{nicefrac}       
\usepackage{microtype}      
\usepackage{enumitem}       
\usepackage{subcaption}

\newcommand{\tr}{^{\top}}
\newcommand{\gammab}{\gamma_{\sf bw}}
\newcommand{\gammabr}{\gamma_{\sf bw,r}}

\hypersetup{colorlinks=true,citecolor=darkgray,linkcolor=black,urlcolor=black}

\newcommand{\R}{\mathbb{R}}

\newcommand{\Z}{\mathbb{Z}}
\newcommand{\A}{\mathcal{A}}

\newcommand{\N}{\mathbb{N}}

\newcommand{\X}{\mathcal{S}}

\newcommand{\mcM}{\mathcal{M}}

\newcommand{\E}{\mathbb{E}}
\newcommand{\U}{\mathcal{U}}

\newcommand{\mcU}{\mathcal{U}}

\newcommand{\opt}{^{\star}}

\theoremstyle{plain}
\newtheorem{theorem}{Theorem}[section]
\newtheorem{lemma}[theorem]{Lemma}
\newtheorem{example}[theorem]{Example}

\newtheorem{corollary}[theorem]{Corollary}

\newtheorem{proposition}[theorem]{Proposition}
\newtheorem{assumption}[theorem]{Assumption}
\theoremstyle{definition}
\newtheorem{definition}[theorem]{Definition}
\newtheorem{remark}[theorem]{Remark}

\newcommand{\Pib}{\Pi\opt_{\sf bw}}
\newcommand{\Pibr}{\Pi\opt_{\sf bw,r}}

\numberwithin{equation}{section}
\numberwithin{theorem}{section}

\title{Reducing Blackwell and Average Optimality to Discounted MDPs via the Blackwell Discount Factor}
\author{Julien Grand-Cl{\'e}ment \\
  Information System and Operations Management Department, HEC Paris \\
  \texttt{grand-clement@hec.fr}
  \and Marek Petrik \\
  Department of Computer Science, University of New Hampshire \\
\texttt{mpetrik@cs.unh.edu}}

\begin{document}
\maketitle
\begin{abstract}
We introduce the {\em Blackwell discount factor} for Markov Decision Processes (MDPs). Classical objectives for MDPs include discounted, average, and Blackwell optimality. Many existing approaches to computing average-optimal policies solve for discounted optimal policies with a discount factor close to $1$, but they only work under strong or hard-to-verify assumptions such as ergodicity or weakly communicating MDPs. In this paper, we show that when the discount factor is larger than the {\em Blackwell discount factor} $\gammab$, all discounted optimal policies become Blackwell- and average-optimal, and we derive a general upper bound on $\gammab$. The upper bound on $\gammab$ provides the first reduction from average and Blackwell optimality to discounted optimality, {\em without any assumptions}, and new polynomial-time algorithms for average- and Blackwell-optimal policies. Our work brings new ideas from the study of polynomials and algebraic numbers to the analysis of MDPs. Our results also apply to robust MDPs, enabling the first algorithms to compute robust Blackwell-optimal policies.
\end{abstract}
\section{Introduction}\label{sec:introduction}
Markov Decision Processes (MDPs) provide a widely-used framework for modeling sequential decision-making problems~\citep{puterman2014markov}. In a (finite) MDP, the decision maker repeatedly interacts with an environment characterized by a finite set of states and a finite set of available actions. The decision maker follows a {\em policy} that prescribes an action at a state at every period. An instantaneous reward is obtained at every period, depending on the current state-action pair, and the system transitions to the next state at the next period. MDPs provide the underlying model for the applications of reinforcement learning (RL), ranging from healthcare~\citep{gottesman2019guidelines} to game solving~\citep{mnih2013playing} and finance~\citep{deng2016deep}.

There are several optimality criteria that measure a decision maker's performance in an MDP. In {\em discounted optimality}, the decision maker optimizes the discounted return, defined as the sum of the instantaneous rewards over the infinite horizon, where future rewards are discounted with a {\em discount factor} $\gamma \in [0,1)$. In {\em average optimality}, the decision maker optimizes the average return, defined as the average of the instantaneous rewards obtained over the infinite horizon. The average return ignores any return gathered in finite time, i.e., it does not reflect the transient performance of a policy and it only focuses on the steady-state behavior. 

Perhaps the most selective optimality criterion in MDPs is {\em Blackwell optimality}~\citep{puterman2014markov}. A policy is Blackwell-optimal if it optimizes the discounted return simultaneously for all discount factors sufficiently close to $1$. Since a discount factor close to $1$ can be interpreted as a preference for rewards obtained in later periods, Blackwell-optimal policies are also average-optimal. However, average-optimal policies need not be Blackwell-optimal. As such, Blackwell optimality can be a useful criterion in environments with no natural, or known, discount factor. Also, any algorithm that computes a Blackwell-optimal policy also immediately computes an average-optimal policy. This is one of the reasons why better understanding the Blackwell optimality criterion is mentioned as {\em ``one of the pressing questions in RL''} in the list of open research problems from a recent survey on RL for average reward optimality~\citep{dewanto2020average}.

Average-optimal policies can be computed efficiently via linear programming (section 9.3,~\cite{puterman2014markov}). However, virtually all of the recent algorithms for computing average-optimal policies require strong assumptions on the underlying Markov chains associated with the policies in the MDP instance, such as ergodicity~\citep{wang2017primal}, the unichain and aperiodicity properties~\citep{schneckenreither2020average}, weakly communicating MDPs~\citep{wang2022near}, or assumptions on the mixing time associated with any deterministic policies~\citep{jin2020efficiently,jin2021towards}. These assumptions are motivated by technical considerations (e.g., ensuring that the average reward is uniform across all states) and can be restrictive in practice~\citep{puterman2014markov} and NP-hard to verify, such as the unichain property~\citep{tsitsiklis2007np}.

For computing Blackwell-optimal policies, the situation is quite complex: existing methods for computing Blackwell-optimal policies rely on linear programming over the field of Laurent series (power series including negative coefficients)~\citep{smallwood1966optimum,hordijk1985sensitivity}, or on an algorithm based on a nested sequence of optimality equations~\citep{veinott1969discrete,o2017polynomial} which requires to solve multiple linear programs sequentially. The intricacy of these algorithms makes them difficult to use in practice, with no complexity guarantees for the method relying on Laurent series, and no known public implementation for the method based on nested optimality equations.

In summary, existing algorithms that compute average-optimal policies require restrictive assumptions, and algorithms that compute Blackwell-optimal policies are very complicated. This situation is in stark contrast with the vast literature on solving discounted MDPs, where multiple efficient, general, and well-understood methods exist, including value iteration, policy iteration, and linear programming~(chapter 6,~\cite{puterman2014markov}).
This is the starting point of this paper, which aims to develop new algorithms for computing average-optimal and Blackwell-optimal policies through a reduction to discounted MDPs. 

Our {\bf main contributions} can be summarized as follows. 

{\em We introduce the Blackwell discount factor}, which is a discount factor $\gammab \in [0,1)$ such that any discounted optimal policy for a larger discount factor is also Blackwell-optimal. In other words, discounted optimality for $\gamma > \gammab$ is \emph{sufficient} for Blackwell optimality. This is important because knowing $\gammab$ makes it straightforward to compute Blackwell- and average-optimal policies: solving a discounted MDP with a discount factor of $\gamma \in (\gammab,1)$ returns a Blackwell-optimal policy, and fast general algorithms exist for solving discounted MDPs. 
In contrast, prior work has focused on {\em necessary} condition for Blackwell optimality and it has often been overlooked that even if a Blackwell-optimal policy remains discounted optimal for large enough discount factors, other policies may be discounted optimal but may not be Blackwell-optimal.
The classical approach to Blackwell optimality plays an important role in theoretical analysis but, as we argue, cannot be used to compute Blackwell-optimal policies with simple algorithms. As our first main contribution, we show that $\gammab$ always exists for finite MDPs.

{\em Upper bound the Blackwell discount factor.} As our second main contribution, we provide a strict upper bound on $\gammab$ given an MDP instance with rational entries, i.e., assuming that the instantaneous rewards and the transition probabilities of the MDP are rational numbers. Solving a discounted MDP with a discount factor larger or equal than our strict upper bound
returns a Blackwell-optimal policy. Crucially, our strict upper bound does not require any assumptions on the underlying structure of the MDP, which is a significant improvement on existing literature. 
Interestingly, the construction of our upper bound relies on novel techniques for analyzing MDPs. We interpret $\gammab \in [0,1)$ as the root of a polynomial equation $p(\gamma)=0$ in $\gamma$, show $p(1)=0$, and use a lower bound ${\sf sep}(p)$ on the distance between any two roots of a polynomial $p$, known as the {\em separation of algebraic numbers}. This shows that $\gammab < 1 - {\sf sep}(p)$, where ${\sf sep}(p)$ depends on the MDP instance. Since Blackwell optimality implies average optimality, we also obtain the first reduction from average optimality to discounted optimality, {\em without any assumption} on the MDP structure.

{\em Blackwell discount factor for robust MDPs.} 
We consider the case of robust reinforcement learning where the transition probabilities are unknown and, instead, belong to an uncertainty set. As our third main contribution, we show that the robust Blackwell discount factor $\gammabr$ exists for popular models of uncertainty, such as sa-rectangular robust MDPs with polyhedral uncertainty~\citep{iyengar2005robust,goyal2022robust}. For this setting, we generalize our upper bound on $\gammab$ for MDPs to an upper bound on $\gammabr$ for robust MDPs. Since robust MDPs with discounted optimality can be solved via value iteration and policy iteration, we provide the first algorithms to compute Blackwell-optimal policies for robust MDPs.

We conclude this section with a discussion on {\bf related works}. Several papers study conditions under which it is possible to compute an average-optimal policy by computing a discounted optimal policy for sufficiently large discount factors. To the best of our knowledge, all existing results require strong assumptions to obtain such a reduction. The earliest attempt in this direction can be traced back to \cite{ross1968non}, which assumes that all transition probabilities are bounded from below by $\epsilon>0$. Recent extensions of this result assume bounded times of first returns~\citep{akian2013policy,huang2016complexity}, or the related condition that the MDP is weakly-communicating~\citep{wang2022near}. Closer to our work, 
\cite{friedmann2011exponential,zwick1996complexity}  and \cite{perotto2018tuning} obtain a reduction from average optimality to discounted optimality, but their results require that the transition probabilities are deterministic. Other recent reductions require some assumptions on the mixing times of the Markov chains induced by deterministic policies~\citep{jin2021towards}. \cite{boone2022identification} propose a sampling algorithm to learn a Blackwell-optimal policy, in a special case in which it reduces to bias optimality. Under the unichain assumption, \cite{wang2023robust} show the existence of Blackwell-optimal policies for sa-rectangular robust MDPs~\citep{iyengar2005robust}, which is connected to the existence results in \cite{tewari2007bounded} and \cite{goyal2022robust}. In contrast to the existing literature, we do not need any assumption on the underlying MDP to obtain our reduction from Blackwell optimality and average optimality to discounted optimality.
\section{Preliminaries on MDPs}\label{sec:markov-decision-process}
In this section, we introduce the Markov Decision Process (MDP) framework and optimality criteria related to our work. An MDP instance is characterized by a tuple $\mcM = \left(\X,\A,\bm{r},\bm{P}\right)$, where $\X$ is a finite set of states and $\A$ is a finite set of actions. The instantaneous rewards are denoted by $\bm{r} \in \R^{\X \times \A}$ and the transition probabilities are denoted by $\bm{P} \in \left(\Delta(\X)\right)^{\X \times \A}$, where $\Delta(\X)$ is the simplex over $\X$.

At any time period $t$, the decision maker is in a state $s_{t} \in \X$, chooses an action $a_{t} \in \A$, obtains an instantaneous reward $r_{s_{t}a_{t}} \in \R$, and transitions to state $s_{t+1}$ with probability $P_{s_{t}a_{t}s_{t+1}} \in [0,1]$. A {\em deterministic stationary} policy $\pi\colon \X \rightarrow \A$ assigns an action to each state. Because there exists an optimal deterministic stationary policy for all the criteria considered in this paper~\citep{puterman2014markov}, we simply refer to them as \emph{policies} and denote them as $\Pi = \A^{\X}$.

A policy $\pi \in \Pi$ induces a vector of expected instantaneous reward $\bm{r}_{\pi} \in \R^{\X}$, defined as  $ r_{\pi,s} = r_{s\pi(s)}, \forall \; s \in \X$, as well as a Markov chain over $\X$, evolving via a transition matrix $\bm{P}_{\pi} \in \R^{\X \times \X}$, defined as
$P_{\pi,ss'} = P_{s\pi(s)s'}, \forall \; s,s' \in \X.$

\paragraph{Optimality criteria.}
Given a discount factor $\gamma \in [0,1)$ and a policy $\pi \in \Pi$, the {\em value function} $\bm{v}^{\pi}_{\gamma} \in \R^{\X}$ represents the discounted value obtained starting from each state:
\begin{equation}\label{eq:discounted-value-function}
v_{\gamma,s}^{\pi} = \E^{\pi,\bm{P}} \left[ \sum_{t=0}^{+\infty} \gamma^t r_{s_{t},a_{t}} \; \Big\vert \; s_{0} = s\right], \forall \; s \in \X.
\end{equation}
We start with the definition of discounted optimality, which is the most popular optimality criterion in RL. 
\begin{definition}
Given $\gamma \in [0,1)$, a policy $\pi \in \Pi$ is $\gamma$-discounted optimal if $v_{\gamma,s}^{\pi} \geq v_{\gamma,s}^{\pi'}, \forall \; \pi' \in \Pi, \forall \; s \in \X.$ We call $\Pi\opt_{\gamma} \subset \Pi$ the set of $\gamma$-discounted optimal policies. 
\end{definition}
The discount factor $\gamma \in [0,1)$ represents the preference of the decision maker for current rewards compared to rewards obtained in the later periods. The difficulty of choosing the discount factor $\gamma$ for a specific RL application is well recognized~\citep{tang2021taylor}. In some applications, it is reasonable to choose values of $\gamma$ close to $1$, e.g., in financial applications~\citep{deng2016deep}, in healthcare applications~\citep{neumann2016cost,garcia2021interpretable} or when solving games using reinforcement learning algorithms~\citep{brockman2016openai}. In other applications, $\gamma$ is merely treated as a parameter introduced artificially for algorithmic purposes, e.g., for controlling the variance of the policy gradient estimates~\citep{baxter2001infinite}, or for ensuring convergence of the learning algorithms. In particular, a discounted optimal policy can be computed efficiently with value iteration, policy iteration, and linear programming~\citep{puterman2014markov}. Notably, these algorithms do not require any assumptions on the MDP instance $\mcM$.

Another fundamental optimality criterion is {\em average optimality}. Let us define the average reward $\bm{g}^{\pi} \in \R^{\X}$ of a policy $\pi \in \Pi$ as
\[
  g_{s}^{\pi}  = \lim_{T \rightarrow + \infty} \frac{1}{T+1}\E^{\pi,\bm{P}} \left[ \sum_{t=0}^{T}  r_{s_{t},a_{t}} \; \Big\vert \; s_{0} = s\right], \forall \; s \in \X.
\]
A policy $\pi$ is average-optimal if $\bm{g}^{\pi} \geq \bm{g}^{\pi'}, \forall \; \pi' \in \Pi.$ 
Average optimality has been extensively studied in the RL literature, as it alleviates the introduction of a potentially artificial discount factor. Classical algorithms include relative value iteration~\citep{jalali1990distributed,yang2016efficient,dong2019q}, and gradient-based methods~\citep{bhatnagar2007incremental,iwaki2019implicit}.
We refer the reader to \cite{dewanto2020average} for an extensive survey on RL algorithms for computing average-optimal policies. 

Despite its natural interpretation, several technical complications arise from considering average optimality instead of discounted optimality.
In all generality, the average reward $\bm{g}^{\pi}$ of a policy is not even a continuous function of the policy $\pi$ (e.g., chapter 4, \cite{feinberg2012handbook}). This can make gradient-based methods inefficient, since a small change in the policy may result in drastic changes in the average reward. Additionally, the Bellman operator associated with the average optimality criterion is not a contraction and may have multiple fixed points. These complications can be circumvented by assuming several structural properties on the MDP instance $\mcM$, such as bounded times of first returns and weakly-communicating MDPs~\citep{akian2013policy,wang2022near}. Some of these assumptions may be hard to verify in a simulation environment where only samples are available, or NP-hard to verify even when the MDP instance is fully known, as is the case for the unichain assumption~\citep{tsitsiklis2007np}. One of our goals in this paper is to provide a method to compute average-optimal policies via solving discounted MDPs. We will do so via the notion of {\em Blackwell optimality}.
\section{Classical Blackwell optimality}\label{sec:blackwell-optimality}

In this section, we describe the classical definition of Blackwell optimality in MDPs and summarize its main limitations. Section~\ref{sec:blk-definition} gives this definition of a Blackwell-optimal policy and outlines the proof of its existence. This proof will serve as a building block of our main result in Section~\ref{sec:blackwell-discount-factor}. We highlight the main limitations of the existing definition of Blackwell optimality in Section~\ref{sec:structural-properties}.

\subsection{Existing definition and algorithms}\label{sec:blk-definition}
We now give the classical definition of Blackwell optimality, which provides an interesting connection between discounted optimality and average optimality.
\begin{definition}\label{def:blackwell-optimality}
A policy $\pi$ is {\em Blackwell-optimal} if there exists $\gamma \in [0,1)$, such that 
$ \pi \in \Pi\opt_{\gamma'}, \; \forall \, \gamma' \in [\gamma,1).$
We call $\Pib$ the set of Blackwell-optimal policies.
\end{definition}
In short, a Blackwell-optimal policy is $\gamma$-discounted optimal for all discount factors $\gamma$ sufficiently close to $1$. This notion dates back to \cite{blackwell1962discrete} and it has become popular in the field of reinforcement learning, mainly due to its connection to average optimality~\citep{dewanto2021examining}. Blackwell optimality bridges the gap between the different optimality criteria: Blackwell optimality is defined in terms of discounted optimality, yet Blackwell-optimal policies are average-optimal (theorem 10.1.5, \cite{puterman2014markov}). Therefore, any advances in computing Blackwell-optimal policies transfer to advances in computing average-optimal policies. 
\paragraph{Existence of a Blackwell-optimal policy.}
A Blackwell-optimal policy is guaranteed to exist for finite MDPs with $|\X|<\infty$ and $|\A|<\infty$.
\begin{theorem}[\cite{blackwell1962discrete}]\label{th:blackwell-opt-mdp}
In any finite MDP, there exists at least one Blackwell-optimal policy: $\Pib \neq \emptyset$.
\end{theorem}
We now highlight the main steps of a proof of Theorem~\ref{th:blackwell-opt-mdp} based on section 10.1.1 in \cite{puterman2014markov}. Summarizing this proof is important because it is not well-known and serves as a building block for our results.

{\em Step 1.} The first step of the proof of Theorem~\ref{th:blackwell-opt-mdp} is to show that for any two policies $\pi,\pi' \in \Pi$ and any state $s \in \X$, the function $\gamma \mapsto v_{\gamma,s}^{\pi}-v_{\gamma,s}^{\pi'}$ only has finitely many zeros in $[0,1)$.
 This is a consequence of the following lemma.
\begin{lemma}\label{lem:value-function-rational}
For $\pi \in \Pi$ and $s \in \X$, $\gamma \mapsto v_{\gamma,s}^{\pi}$ is a {\em rational function} on $[0,1)$, i.e., it is the ratio of two polynomials. 
\end{lemma} 
Lemma~\ref{lem:value-function-rational} follows from the Bellman equation for the value function $\bm{v}^{\pi}$: $\bm{v}^{\pi} = \bm{r}_{\pi} + \gamma \bm{P}_{\pi}\bm{v}^{\pi}$.
Therefore, $\bm{v}^{\pi}$ is the unique solution to the equation $\bm{Ax}=\bm{b}$, for $\bm{b}=\bm{r}_{\pi}$ and $\bm{A} = \bm{I} -\gamma \bm{P}_{\pi}$.
Lemma~\ref{lem:value-function-rational} then follows directly from Cramer's rule for the solution of a system of linear equations: since $\bm{A}$ is invertible, then $\bm{Ax}=\bm{b}$ has a unique solution $\bm{x}$, which satisfies $x_{s} = \det(\bm{A}_{s})/\det(\bm{A}), \forall \; s \in \X$, with $\det(\cdot)$ the determinant of a matrix and $\bm{A}_{s}$ the matrix formed by replacing the $s$-th column of $\bm{A}$ by the vector $\bm{b}$. A consequence of Lemma~\ref{lem:value-function-rational} is the following.
\begin{corollary}\label{cor:rational-difference-value}
For any two policies $\pi,\pi'$ and any state $s \in \X$, the function $\gamma \mapsto v_{\gamma,s}^{\pi}-v_{\gamma,s}^{\pi'}$ is a rational function. 
\end{corollary}
Since $\gamma \mapsto v_{\gamma,s}^{\pi}-v_{\gamma,s}^{\pi'}$ is rational, its zeros are the zeros of a polynomial. Therefore,$\gamma \mapsto v_{\gamma,s}^{\pi}-v_{\gamma,s}^{\pi'}$ is either identically equal to $0$, or it has only has finitely many roots in $[0,1)$.

{\em Step 2.} We can now conclude the proof of Theorem~\ref{th:blackwell-opt-mdp} as follows. For any pair of policies $\pi,\pi' \in \Pi$ and any state $s \in \X$ such that $\gamma \mapsto v_{\gamma,s}^{\pi}-v_{\gamma,s}^{\pi'}$ 
 is not identically equal to $0$, we write $\gamma(\pi,\pi',s) \in [0,1)$ for the largest zero of the map $\gamma \mapsto v_{\gamma,s}^{\pi}-v_{\gamma,s}^{\pi'}$ in $[0,1)$:
 \begin{equation}\label{eq:gamma-pi-pi-prime-s}
     \gamma(\pi,\pi',s) = \max \{\gamma \in [0,1) | v_{\gamma,s}^{\pi}-v_{\gamma,s}^{\pi'}=0\}.
 \end{equation}
 We let $\gamma(\pi,\pi',s)=0$ if $\gamma \mapsto v_{\gamma,s}^{\pi}-v_{\gamma,s}^{\pi'}$ is identically equal to $0$ on the entire interval $[0,1)$.  We now let
 \begin{equation}\label{eq:bar-gamma}
 \bar{\gamma} = \max_{\pi,\pi' \in \Pi,s \in \X} \gamma(\pi,\pi',s).
 \end{equation}
We have $\bar{\gamma}<1$ since there is a finite number of (stationary, deterministic) policies and a finite number of states. Let $\pi$ be $\gamma$-discounted optimal for a certain $\gamma > \bar{\gamma}$. We have, for any $ s \in \X, v^{\pi}_{\gamma,s} \geq v^{\pi'}_{\gamma,s}, \forall \; \pi' \in \Pi.$
By the definition of $\bar{\gamma}$, the map $\gamma \mapsto v^{\pi}_{\gamma,s} - v^{\pi'}_{\gamma,s}$ cannot change a sign on $[\bar{\gamma},1)$ (because it cannot be equal to $0$), for any policy $\pi' \in \Pi$ and any state $s \in \X$, i.e., we have $v^{\pi}_{\gamma',s} \geq v^{\pi'}_{\gamma',s}, \forall \; \pi' \in \Pi, \forall \; \gamma' \in (\gamma,1).$
This shows that $\pi$ remains $\gamma'$-discounted optimal for all $\gamma' > \gamma$, and, therefore, $\pi$ is Blackwell-optimal.
\paragraph{Existing algorithms.}
To the best of our knowledge, there are only two algorithms to compute a Blackwell-optimal policy.
The first algorithm~\citep{smallwood1966optimum,hordijk1985sensitivity} formulates MDPs with varying discount factors as linear programs over the field of power series with potentially negative coefficients, known as Laurent series.
This generalizes the observation that MDPs with a fixed discount factor can be formulated as linear programs over $\R^{\X}$. An implementation of the simplex method for solving linear programs over power series explores the entire interval $[0,1)$ and computes the subintervals of $[0,1)$ where an optimal policy can be chosen constant (as a function of $\gamma$). It returns a Blackwell-optimal policy in a finite number of operations. However, there are no complexity guarantees for this algorithm. The second algorithm is based on a set of $(|\X|+1)$-nested equations indexed by $n=-1,...,|\X|-1$, which need to be solved sequentially by solving three linear programs at each stage $n$~\citep{o2017polynomial}. This gives a polynomial-time algorithm for computing Blackwell-optimal policies, requiring solving $3(|\X|+1)$ linear programs of dimension $O\left(|\X|\right)$. A simpler version of this algorithm is in section 10.3.4 in \cite{puterman2014markov}, but only finite convergence is proved. To the best of our knowledge, there are no implementations of these algorithms available.
\subsection{Limitations of existing  approaches}\label{sec:structural-properties}
We now emphasize the limitations of the classical definition and algorithms for computing Blackwell-optimal policies.

First, Definition~\ref{def:blackwell-optimality} only leads to algorithms that are significantly more involved than the method for solving discounted MDPs. In particular, the two existing algorithms for computing Blackwell-optimal policies require the handling of complex objects, e.g., the field of power series and nested optimality equations involving multiple subproblems that need to be solved sequentially. The intricacy of both algorithms makes them difficult to implement. In Section~\ref{sec:blackwell-discount-factor}, we introduce the notion of the {\em Blackwell discount factor}, which provides a reduction of Blackwell optimality to discounted optimality, leading to algorithms for computing Blackwell-optimal policies that are conceptually much simpler.

Second, Definition~\ref{def:blackwell-optimality} implicitly introduces, for each Blackwell-optimal policy $\pi \in \Pib$, a discount factor $\gamma(\pi) \in [0,1)$, defined as the smallest discount factor after which $\pi$ remains discounted optimal:
\begin{equation}\label{eq:definition-bar-gamma-pi}
\gamma(\pi) = \min \{ \gamma \in [0,1) \; |  \; \pi \in \Pi\opt_{\gamma'}, \forall \; \gamma' \in [\gamma,1)\}.
\end{equation}
However, this discount factor $\gamma(\pi) \in [0,1)$ does not provide a method to compute a Blackwell-optimal policy, as the following proposition shows.
\begin{proposition}\label{prop:limit-def}
  There exists an MDP instance $\mcM$, a Blackwell-optimal policy $\pi\in \Pib$, and discount factors $\gamma_1, \gamma_2 \in [0,1)$ with $\gamma_1 < \gamma(\pi) < \gamma_2$ such that:
  \begin{enumerate}[nosep]
    \item the policy $\pi$ is $\gamma_1$-discounted optimal, and
    \item there exists $\pi' \neq \pi$ that is $\gamma_2$-discounted optimal and \emph{not} Blackwell-optimal.
  \end{enumerate}
\end{proposition}
Proposition~\ref{prop:limit-def} shows that solving a $\gamma$-discounted MDP for discount factor $\gamma > \gamma(\pi)$ does not compute a Blackwell-optimal policy: the policy $\pi'$ in Proposition~\ref{prop:limit-def} is optimal for $\gamma_2>\gamma(\pi)$ but is not Blackwell- or average-optimal. 
It also shows that $\gamma(\pi)$ is not the smallest discount factor for which $\pi$ is discounted optimal. Overall, Proposition~\ref{prop:limit-def} shows that the discount factor $\gamma(\pi)$, appearing in the classical definition of Blackwell optimality, cannot be exploited to compute a Blackwell-optimal policy.

The proof of Proposition~\ref{prop:limit-def} is based on the next example.
\begin{example}\label{ex:counter-example-0}
We consider the MDP instance from Figure~\ref{fig:example-0}. The decision maker starts in state $0$ and chooses one of three actions $\{a_{1},a_{2},a_{3}\}$; there is no choice in other states, all transitions are deterministic, and the rewards are indicated above the transition arcs. The reward for $a_{1}$ is $1$ and the process transitions to the absorbing state $7$, which gives a reward of $0$. The reward for $a_2$ is $0$, and the process transitions to states $1,2,3$ before reaching the absorbing state $7$. Therefore, the value function $v^{a_{2}}_{\gamma}$ is $v^{a_{2}}_{\gamma}=r_{1} \gamma + r_{2} \gamma^{2}$. Similarly, we have $v^{a_{3}}_{\gamma}=r_{4} \gamma + r_{5} \gamma^{2}$. Meanwhile, the value function $v^{a_{1}}_{\gamma}$ is always equal to $1$.  By choosing $(r_{1},r_{2}) = (6,-8)$ and $(r_{4},r_{5}) = (8/3,-16/9)$, we obtain the value functions represented in Figure \ref{fig:example-1}. In particular, $v_{\gamma}^{a_{2}}$ is the parabola that is equal to $0$ at $\gamma=0$, and equal to $1$ at $\gamma \in \{1/4,1/2\}$, and $v_{\gamma}^{a_{3}}$ is the parabola that is equal to $0$ at $\gamma=0$ and equal to its maximum $1$ at $\gamma = 3/4$. This shows that $a_{1}$ is Blackwell-optimal with $\gamma(a_{1})=1/2$. Additionally, for $\gamma_{1} \in [0,1/4]$, $a_{1}$ is $\gamma_{1}$-discounted optimal. Finally, $a_{3}$ is $\gamma_{2}$-discounted optimal for $\gamma_{2} = 3/4$, but it is not Blackwell-optimal.

\begin{figure}[htp]
\begin{center}
  \begin{subfigure}{0.34\textwidth}
    \centering
\includegraphics[width=1.0\linewidth]{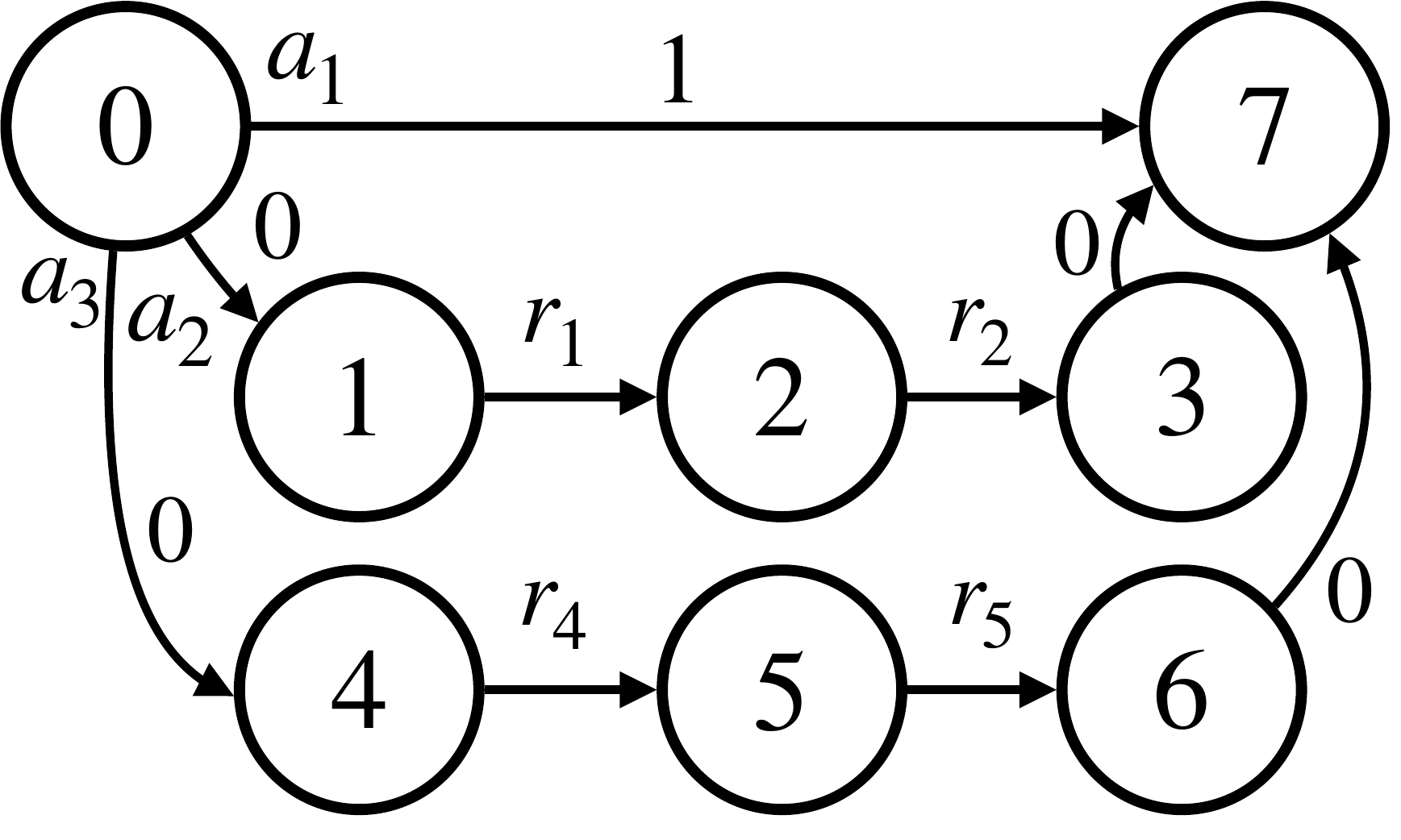}
  \vspace{3em}
\caption{}\label{fig:example-0}
\end{subfigure}
\hspace{3ex}
\begin{subfigure}{0.44\textwidth}
  \centering
  \includegraphics[width=1.0\linewidth]{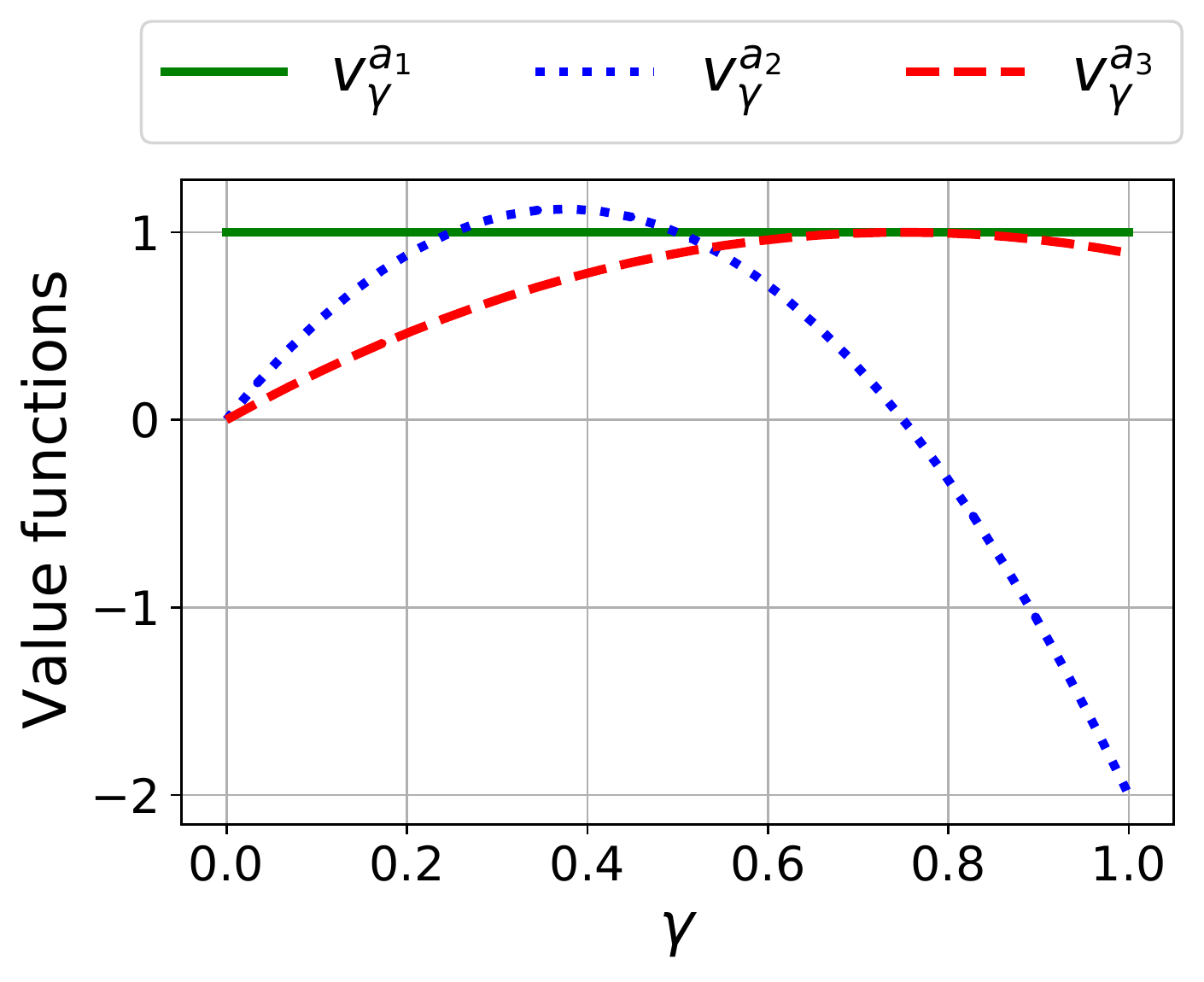}
\caption{}\label{fig:example-1}
\end{subfigure}
\end{center}
\vspace{-1.5em}
\caption{MDP instance for Example~\ref{ex:counter-example-0} (Figure~\ref{fig:example-0}) . There are three actions in state $0$ and the transitions are deterministic. The instantaneous rewards are represented above the transition arcs. The value functions are represented in Figure~\ref{fig:example-1}.}
\end{figure}
\end{example} 

The following theorem further strengthens Proposition~\ref{prop:limit-def} to show that there exists an MDP with only two different policies, but where a Blackwell-optimal policy may be $\gamma$-optimal for $\gamma$ in an {\em arbitrary} number of {\em arbitrary} disjoint subintervals of $[0,1)$.
 \begin{theorem}\label{th:interval-optimal-policies}
For any odd integer $N \in \N$ and any sequence 
$ 0 = \gamma_{0} < \gamma_{1}<...<\gamma_{N-1} < \gamma_{N} = 1,$
there exists an MDP instance $\left(\X,\A,\bm{r},\bm{P}\right)$ with $|\X| = N+1$ and $|\A|=2$, and 
two policies $\pi_{1},\pi_{2}$ such that $\pi_{1}$ is the unique optimal policy on any of the intervals $(\gamma_{2i},\gamma_{2i+1})$ for $i=0,...,(N-1)/2$ and $\pi_{2}$ is the unique optimal policy on $(\gamma_{2i-1},\gamma_{2i})$, for $i=1,...,(N-1)/2$.
\end{theorem}
Theorem~\ref{th:interval-optimal-policies} follows from the following example.
\begin{example}\label{ex:multiple-intervals}
    Consider the following MDP instance, represented in Figure~\ref{fig:example-multiple-intervals}.
The initial state is state $0$, where there are two actions to be chosen, $a_{1}$ or $a_{2}$. Action $a_{1}$ yields an instantaneous reward of $1$ and then the decision maker transitions to the absorbing state $N$, where there is a reward of $0$. Otherwise, choosing action $a_{2}$  yields an instantaneous reward $r_{0}$ and takes the decision maker through a deterministic sequence of states $1,...,N-1$ with rewards $r_{1},...,r_{N-1}$, before transitioning to state $N$. 
For a given $\gamma \in [0,1)$, the closed-form expressions for the value functions $v^{a_{1}}_{\gamma},v^{a_{2}}_{\gamma}$ are $v^{a_{1}}_{\gamma}  = 1$ and $v^{a_{2}}_{\gamma}  = \sum_{t=0}^{N-1} r_{t}\gamma^{t}$.

Note that $\gamma \mapsto v^{a_{2}}_{\gamma}$ is a polynomial of degree $N-1$. Using Lagrange interpolation polynomials (section 0.9.11, \cite{horn2012matrix}), we can find coefficients $r_{0},...,r_{N-1}$ such that $\gamma \mapsto v^{a_{1}}_{\gamma}$ is equal to $1$ for all $N-1$ discount factors $\gamma_{1},..., \gamma_{N-1}$ and equal to $0.9$ at $\gamma_{0}=0$.
The value function $v_{\gamma}^{a_{2}}$ resulting from this construction is highlighted in Figure~\ref{fig:value-functions-multiple-intervals} for $N=5$ and $(\gamma_{0},\gamma_{1},\gamma_{2},\gamma_{3},\gamma_{4},\gamma_{5})=(0,0.2,0.4,0.6,0.8,1.0)$.
Let us note $q\colon\gamma \mapsto v^{a_{1}}_{\gamma} - v^{a_{2}}_{\gamma}$. Our choice of the rewards ensures that $q$ is a polynomial of degree $N-1$, with $q(0)>0$, and $q(\gamma)=0$ for $\gamma \in \{\gamma_{1}, ..., \gamma_{N-1}\}$. Because $\gamma \mapsto q(\gamma)-1$ is a polynomial of degree $N-1$ with $N-1$ different real roots, it changes signs at every root.  This shows that  
 $\gamma \mapsto v^{a_{1}}_{\gamma} - v^{a_{2}}_{\gamma}$ is positive on $(\gamma_{0},\gamma_{1})$, negative on $(\gamma_{1},\gamma_{2})$, then positive on $(\gamma_{2},\gamma_{3})$, etc.. Action $a_{1}$ is optimal on $(\gamma_{N-1},\gamma_{N}) = (\gamma_{N-1},1)$ because $N$ is odd. This concludes the proof of Theorem~\ref{th:interval-optimal-policies}.

\begin{figure}[htp]
\begin{center}
  \begin{subfigure}{0.34\textwidth}
    \centering
    \includegraphics[width=1.0\linewidth]{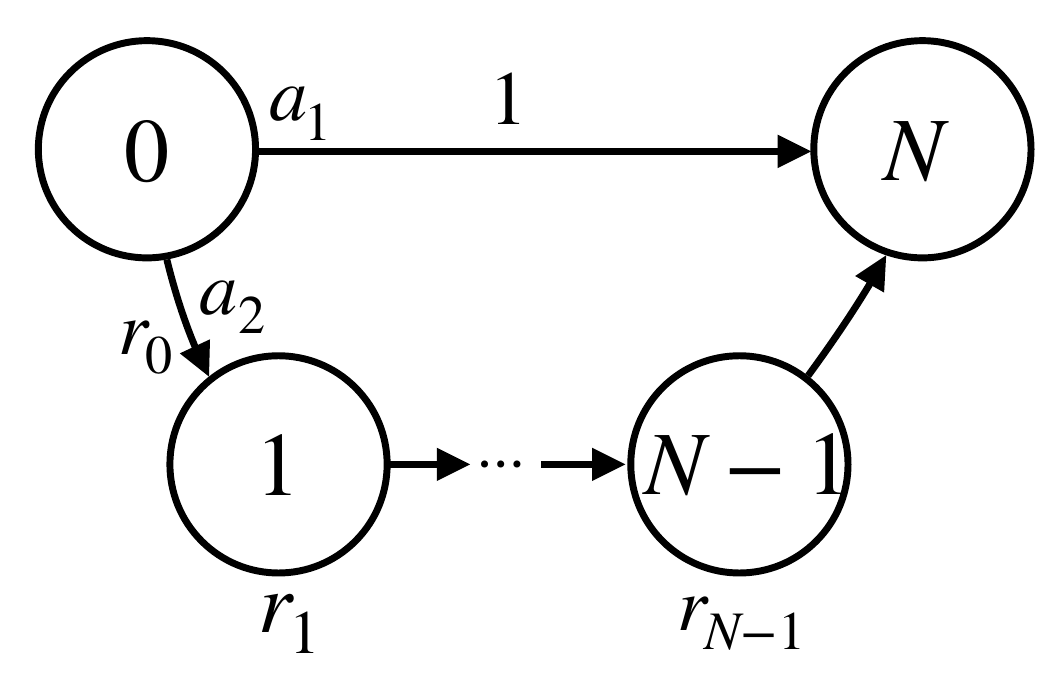}
    \vspace{0.1em}
\caption{}\label{fig:example-multiple-intervals}
\end{subfigure}
\hspace{3ex}
\begin{subfigure}{0.44\textwidth}
  \centering
\includegraphics[width=1.0\linewidth]{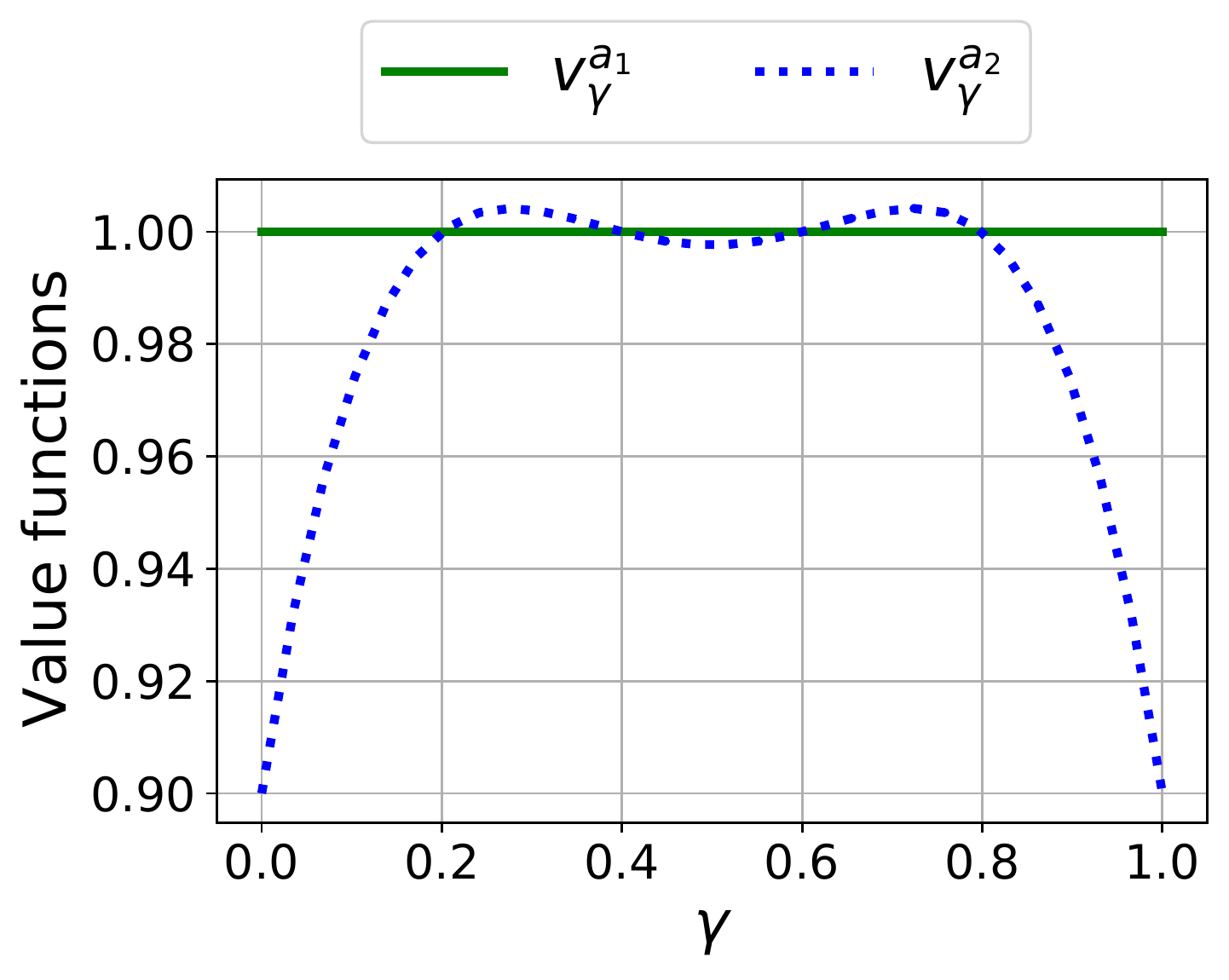}
\caption{}\label{fig:value-functions-multiple-intervals}
\end{subfigure}
\end{center}
\vspace{-1.5em}
\caption{MDP instance for Example~\ref{ex:multiple-intervals} (Figure~\ref{fig:example-multiple-intervals}) and the value functions for $N=5$ (Figure~\ref{fig:value-functions-multiple-intervals}).}
\end{figure}
\end{example} 



\section{Blackwell discount factor}
\label{sec:blackwell-discount-factor}

 In the previous section, we have seen that the classical definition of Blackwell optimality does not lead to simple algorithms to compute a Blackwell-optimal policy. Our main contribution in this section is to introduce the notion of the {\em Blackwell discount factor}, which we use to construct a reduction from Blackwell optimality and average optimality to the discounted optimality criterion. 
This will provide algorithms to compute Blackwell-optimal and average policies that are much simpler than the methods discussed in the previous section.

Intuitively, the Blackwell discount factor $\gammab \in [0,1)$ is a discount factor sufficiently close to $1$ such that any discounted optimal policy is also Blackwell optimal.  
\begin{definition}\label{def:blackwell-discount-factor-2}
    The Blackwell discount factor $\gammab \in [0,1)$ equals to
\begin{equation*}
\gammab = \inf \{ \gamma \in [0,1) \mid   \Pi\opt_{\gamma'}=\Pib, \forall \; \gamma' \in (\gamma,1)\},
\end{equation*}
  where $\Pib$ is the set of Blackwell-optimal policies.
\end{definition}
We first show the existence of the Blackwell discount factor.
\begin{theorem}\label{th:blackwell-discount-factor}
The Blackwell discount factor $\gammab$ exists in any finite MDP.
\end{theorem}
\begin{proof}  
To show the existence of the Blackwell discount factor, we show that there exists a discount factor $\gamma \in [0,1)$ such that $\Pi\opt_{\gamma'}=\Pib, \forall \; \gamma' \in (\gamma,1)$. 
Let $\bar{\gamma}$ defined as in Equation~\eqref{eq:bar-gamma}.
We will show that $\bar{\gamma}$ satisfies: $\forall \; \gamma \in [\bar{\gamma},1), \Pi\opt_{\gamma} = \Pib$. Let $\gamma' \in (\bar{\gamma},1)$ and let $\pi$ be a policy that is $\gamma'$-discounted optimal. By definition, we have 
$ v^{\pi}_{\gamma',s} \geq v^{\pi'}_{\gamma',s}, \forall \; \pi' \in \Pi, \forall \; s \in \X.$
Since $\gamma' > \bar{\gamma}$, the map $\gamma \mapsto v^{\pi}_{\gamma,s} - v^{\pi'}_{\gamma,s}$ does not change sign on $[\bar{\gamma},1)$. This shows that $\pi$ is also $\gamma$-discounted optimal for all $\gamma \in (\bar{\gamma},1)$. Therefore, $\pi$ is Blackwell optimal. This shows that any $\gamma$-discounted optimal policy is Blackwell optimal, for any $\gamma \in (\bar{\gamma},1)$.
\end{proof}
\begin{remark}
  Our proof of Theorem \ref{th:blackwell-discount-factor} shows that we always have $\gammab \leq \bar{\gamma}$, with $\bar{\gamma}$ defined as in Equation~\eqref{eq:bar-gamma}. This upper bound is tight, since Example \ref{ex:counter-example-0} shows an MDP instance where we have $\gammab = \bar{\gamma}$.
\end{remark}
\paragraph{Difference from the existing definition.}
It is important to elaborate on the difference between Definition~\ref{def:blackwell-optimality} (classical definition of Blackwell optimality) and Definition~\ref{def:blackwell-discount-factor-2} (Blackwell discount factor).

The distinction between $\gammab$ and $\gamma(\pi)$ has often been overlooked in the literature, where it is common to find statements that suggest that $\gamma > \gamma(\pi)$ implies Blackwell optimality of all discounted optimal policies, e.g. in \cite{dewanto2021examining,wang2023robust}. To the best of our knowledge, we are the first to properly introduce the Blackwell discount factor $\gammab$, to show its sufficiency to compute Blackwell-optimal policies, and to clarify the distinction from the definition relying on $\gamma(\pi)$.

In particular, in Definition~\ref{def:blackwell-optimality}, a Blackwell-optimal policy $\pi$ is optimal for any $\gamma \in [\gamma(\pi),1)$. However, for some $\gamma \in [\gamma(\pi),1)$, there may be other optimal policies that are not Blackwell-optimal, as shown in Proposition~\ref{prop:limit-def}. We show an MDP instance like this in Example~\ref{ex:counter-example-0}, where $\gammab = 3/4$ but where $\gamma(a_{1})=1/2$, and $a_{1}$ is the only Blackwell-optimal policy. This shows that in all generality, we may have $\gamma(\pi) < \gammab$, and $\gamma(\pi) \neq \gammab$. 
  \begin{remark}
The authors in \cite{dewanto2020average,dewanto2021examining} also introduce the notation ``$\gammab$'' but they use it to denote $\gamma(\pi)$.
\end{remark}
\paragraph{Reduction to discounted optimality.}
If $\gammab$ is known for a given MDP instance, it is straightforward to compute a Blackwell-optimal policy: we simply solve a discounted MDP with a discount factor $\gamma > \gammab$. Therefore, the notion of Blackwell discount factor provides a method to reduce the criterion of Blackwell optimality and average optimality to the well-studied criterion of discounted optimality. As we have discussed before, efficient methods for solving discounted MDPs such as value iteration, policy iteration, or linear programming have been extensively studied. These algorithms are much simpler than the two existing algorithms for computing Blackwell-optimal policies.
Note that it is enough to compute an upper bound on $\gammab$. In particular, if we are able to show that $\gammab < \gamma'$ for some $\gamma' \in [0,1)$, then following the definition of $\gammab$, we can compute a Blackwell-optimal policy by solving a discounted MDP with a discount factor $\gamma = \gamma'$.
 Therefore, in the rest of Section~\ref{sec:blackwell-discount-factor}, we focus on obtaining an upper bound on $\gammab$.
\subsection{Upper bound on $\gammab$}\label{sec:bound-blackwell-discount-factor}
We now obtain an instance-dependent upper bound on $\gammab$, i.e., we construct a scalar $\eta(\mcM) \in [0,1)$ for each MDP instance $\mcM = (\X,\A,\bm{r},\bm{P})$, such that $\gammab < 1 - \eta(\mcM)$. 
Our main contribution in this section is Theorem~\ref{th:bound-gammab}, which gives a closed-form expression for $\eta(\mcM)$ as a function of the parameters of the MDP $\mcM$ with rational entries.
\begin{assumption}\label{ass:rational-probabilities} 
There exists an $m \in \N$, such that for any $(s,a,s') \in \X \times \A \times \X$, we have $P_{sas'} = n_{sas'}/m$, for $n_{sas'} \in \N, n_{sas'} \leq m$, and $r_{sa} = q_{sa}/m, |q_{sa}| \leq r_{\infty}$.
\end{assumption}
Our main result in this section is the following theorem.
\begin{theorem}\label{th:bound-gammab}
For any MDP instance $\mcM$ satisfying Assumption~\ref{ass:rational-probabilities}, we have $\gammab < 1-\eta(\mcM)$, with
\begin{align*}
\eta(\mcM) & =\frac{1}{2N^{N/2+2}\left(L+1\right)^{N}},\\
N & = 2 |\X|-1,
L  = 2 \cdot |\X| \cdot r_{\infty} \cdot m^{2|\X|} \cdot 4^{|\X|}.
\end{align*}
\end{theorem}

Our proof uses ideas that are new in the MDP literature. We provide an outline of the proof below and defer the full statement to Appendix~\ref{app:bound-blackwell}.

In the first step, by carefully inspecting the proofs of Theorem~\ref{th:blackwell-opt-mdp} and of Theorem~\ref{th:blackwell-discount-factor}, we note that an upper bound for $\gammab$ is $\bar{\gamma}$, as defined in~\eqref{eq:bar-gamma}: 
$\bar{\gamma} = \max_{\pi,\pi' \in \Pi,s \in \X} \gamma(\pi,\pi',s),$
where for $\pi,\pi' \in \Pi$ and $s \in \X$, $\gamma(\pi,\pi',s)$ is the largest discount factor $\gamma$ in $[0,1)$ for which 
 $v^{\pi}_{\gamma,s} - v^{\pi'}_{\gamma,s}=0$ 
when $\gamma \mapsto v^{\pi}_{\gamma,s} - v^{\pi'}_{\gamma,s}$ is not identically equal to $0$, and $0$ otherwise.
Therefore, we focus on obtaining an upper bound on $\gamma(\pi,\pi',s)$ for any two (stationary, deterministic) policies $\pi,\pi' \in \Pi$ and any state $s \in \X$.

In the second step, following Corollary~\ref{cor:rational-difference-value}, the value functions $\gamma \mapsto v^{\pi}_{s},\gamma \mapsto v^{\pi'}_{s}$ are rational functions, i.e., they are ratios of two polynomials. Therefore, we interpret $v^{\pi}_{\gamma,s} - v^{\pi'}_{\gamma,s}=0$ as a polynomial equation in $\gamma$, i.e., as $p(\gamma)=0$ for a certain polynomial $p$. With this notation, $\gamma(\pi,\pi',s) \in [0,1)$ is a root of $p$. We show that $\gamma=1$ is always a root of $p$, even though value functions are a priori not defined for $\gamma=1$.  We then precisely characterize the degree $N$ and the sum $L$ of the absolute values of the coefficients of the polynomial $p$, depending on the MDP instance $\mcM$. In particular, we prove the following theorem.
\begin{theorem}\label{th:bound-sum-abs-values}
The polynomial $p$ has degree $N=2|\X|-1$. Moreover, $m^{2|\X|} p$ has integral coefficients. The sum of the absolute values of the coefficients of $m^{2|\X|} p$ is bounded by $L=2 \cdot |\X| \cdot r_{\infty} \cdot m^{2|\X|} \cdot 4^{|\X|}.$
\end{theorem}

In the third step, we lower-bound the distance between any two distinct roots of $p$. To do this, we rely on the following {\em separation bounds of algebraic numbers}.
\begin{theorem}[\cite{rump1979polynomial}]\label{th:separation-algebraic-number}
Let $p$ be a polynomial of degree $N$ with integer coefficients. Let $L$ be the sum of the absolute values of its coefficients. The distance between any two distinct roots of $p$ is strictly larger than $\eta>0$, with
\[
  \eta = \frac{1}{2N^{N/2+2}\left(L+1\right)^{N}}.
\]
\end{theorem}
Recall that $\gamma(\pi,\pi',s)$ and $1$ are two always roots of $p$, with $\gamma(\pi,\pi',s)<1$. Combining Theorem~\ref{th:bound-sum-abs-values} with Theorem~\ref{th:separation-algebraic-number}, we conclude that
$\gamma(\pi,\pi',s) < 1 - \eta(\mcM)$ for $\eta(\mcM)>0$ defined as in Theorem~\ref{th:bound-gammab}. 
Following the definition of $\bar{\gamma}$, this shows that 
$\bar{\gamma} < 1 - \eta(\mcM),$
and therefore 
$\gammab < 1 - \eta(\mcM),$
which concludes our proof of Theorem \ref{th:bound-gammab}.
\subsection{Discussion}
Using Theorem~\ref{th:bound-gammab}, we obtain the first reduction from Blackwell optimality to discounted optimality: solving a discounted MDP with $\gamma \geq 1 - \eta(\mcM)$ returns a Blackwell-optimal policy. Since Blackwell-optimal policies are also average-optimal, as a consequence of our results we also obtain the first reduction from average  optimality to discounted optimality {\em  without any assumptions on the structure of the underlying Markov chains of the MDP.}

We first discuss the {\bf complexity results} for computing a Blackwell-optimal policy using our reduction.
Policy iteration returns a discounted optimal policy in $O\left(\frac{|\X|^2 |\A|}{1-\gamma}\log\left(\frac{1}{1-\gamma}\right)\right)$ iterations~\citep{scherrer2013improved}, but it may be slow to converge when $\gamma=1-\eta(\mcM)$ as in Theorem~\ref{th:bound-gammab}, since $\eta(\mcM)$ may be close to $0$. Various algorithms exist to obtain convergence faster than $O(1/(1-\gamma))$, such as accelerated value iteration~\citep{Goyal2021} and Anderson acceleration~\citep{zhang2020globally}.

Discounted MDPs can be formulated as linear programs, which can be solved in polynomial-time in the input size of the MDPs, e.g., table 4 in~\cite{ye2011simplex}. Since $\log(\eta(\mcM)) = O\left(|\X| \log(r_{\infty}) + |\X|^{2}\log(m)\right)$, interior point-methods solve a discounted MDP with $\gamma = 1-\eta(\mcM)$ in polynomial-time~\citep{ye2005new}. Therefore, we provide a polynomial-time algorithm for computing Blackwell- and  average-optimal policies for any MDP instance.

{\bf Potential improvements} for the upper bound on $\gammab$ obtained in Theorem~\ref{th:bound-gammab} are an important future direction. For instance, the separation bound from Theorem~\ref{th:separation-algebraic-number} holds for any polynomials, and more precise lower bounds could be obtained for the specific polynomial $p$ appearing in the proof of Theorem~\ref{th:bound-gammab}. Additionally, tighter upper bounds could be obtained for specific MDP instances.
\section{The case of robust MDPs}\label{sec:robust-mdps}
In practice, the value function $\bm{v}^{\pi}_{\gamma}$ may be very sensitive to the values of the transition probabilities $\bm{P}$. To emphasize this dependence, in this section we note $\bm{v}^{\pi,\bm{P}}_{\gamma}$ for the value function associated with a policy $\pi$ and a transition probability $\bm{P}$, defined similarly as in~\eqref{eq:discounted-value-function}. Robust MDPs (RMDPs) ameliorate this issue by considering an {\em uncertainty set} $\U$, which can be seen as a plausible region for the transition probabilities $\bm{P} \in \U$. We focus on the case of sa-rectangular MDPs~\citep{iyengar2005robust}, where $\U = \times_{(s,a) \in \X \times \A} \U_{sa}$ for $\U_{sa} \subset \Delta(\X)$.  The worst-case value function $\bm{v}^{\pi,\U}_{\gamma} \in \R^{\X}$ of a policy $\pi$ is defined as $v_{\gamma,s}^{\pi,\U} = \min_{\bm{P} \in \U} v_{\gamma,s}^{\pi,\bm{P}}, \forall \; s \in \X.$
In discounted RMDPs, the goal is to compute a {\em robust discounted  optimal} policy, defined as follows.
\begin{definition}\label{def:discounted-optimality-rmdps}
Given $\gamma \in [0,1)$, a policy $\pi \in \Pi$ is robust $\gamma$-discounted optimal if $v_{\gamma,s}^{\pi,\U} \geq v_{\gamma,s}^{\pi',\U}, \forall \; \pi' \in \Pi, \forall \; s \in \X.$ 
We write $\Pi\opt_{\gamma,{\sf rob}}$ the set of robust $\gamma$-discounted optimal policies.
\end{definition}
Blackwell optimality for RMDPs is studied in \cite{tewari2007bounded,goyal2022robust}, to address the sensitivity of the robust value functions as regards the choice of discount factors. Its connection to average reward RMDPs is discussed in \cite{wang2023robust}.
\begin{definition}\label{def:blackwell-optimality-robust}
    A policy $\pi \in \Pi$ is {\em robust Blackwell-optimal} if there exists $\gamma \in [0,1)$, such that $\pi \in \Pi\opt_{\gamma',{\sf r}}, \forall \; \gamma' \in [\gamma,1)$.
We call $\Pibr$ the set of robust Blackwell-optimal policies.
\end{definition}
\cite{goyal2022robust} shows the existence of a Blackwell-optimal policy for RMDPs, under the condition that $\U$ is sa-rectangular and has finitely many extreme points. This is the case for popular polyhedral uncertainty sets, e.g., when $\U_{sa}$ is based on the $\ell_{p}$ distance, for $p \in \{1,\infty\}$~\citep{iyengar2005robust,ho2018fast,givan1997bounded}:
\begin{equation}\label{eq:sa-rec-uncertainty-set-l-p}
    \U_{sa} = \{ \bm{p} \in \Delta(\X) \; | \; \| \bm{p}-\bm{P}^{0}_{sa} \|_{p} \leq \alpha_{sa} \},
\end{equation}
for some estimated kernel $\bm{P}^{0}$ and some radius $\alpha_{sa}>0.$
\paragraph{Robust Blackwell discount factor.} For RMDPs, we define the robust Blackwell discount factor $\gammabr$ as follows.
\begin{definition}\label{def:blackwell-discount-factor-rob}
    We define the robust Blackwell discount factor $\gammabr \in [0,1)$ as
    \begin{equation*}
    \gammabr = \inf \{ \gamma \in [0,1) \mid   \Pi\opt_{\gamma',{\sf r}}=\Pibr, \forall\gamma' \in (\gamma,1)\}.
  \end{equation*}
\end{definition}
We provide detailed  proof of the existence of the robust Blackwell discount factor in Appendix~\ref{app:proof-rmdps}.
The proof strategy is the same as for the existence of the Blackwell discount factor for MDPs. In particular, we can obtain the same upper bound on $\gammabr$, by studying the values of $\gamma$ for which $\gamma \mapsto v^{\pi,\bm{P}}_{\gamma,s} - v^{\pi',\bm{P}'}_{\gamma,s}$ cancels, for any two policies $\pi,\pi' \in \Pi$ and any two extreme points $\bm{P},\bm{P}'$ of $\U$. Writing $\gamma(\pi,\pi',s,\bm{P},\bm{P}')$ for the largest zero in $[0,1)$ of the function
$\gamma \mapsto v^{\pi,\bm{P}}_{\gamma,s} - v^{\pi',\bm{P}'}_{\gamma,s}$ if it is not identically equal to zero, or $\gamma(\pi,\pi',s,\bm{P},\bm{P}')=0$ otherwise, an upper bound on $\gammabr$ for RMDPs can be computed as $\bar{\gamma}_{{\sf r}}$, defined as
\[\bar{\gamma}_{{\sf r}}= \max_{\pi,\pi' \in \Pi,s \in \X} \max_{ \bm{P},\bm{P}' \in \U_{\sf ext}}  \gamma(\pi,\pi',s,\bm{P},\bm{P}')
 \]
with $\U_{\sf ext}$ the set of extreme points of $\U$. This directly leads to the following theorem.
 \begin{theorem}\label{th:bound-gammab-sa}
     Assume that $\U$ is sa-rectangular with finitely many extreme points, and suppose that Assumption~\ref{ass:rational-probabilities} holds when $\bm{P}$ is replaced by any extreme points of $\U.$ Then $\gammabr \leq 1 -\eta(\mcM)$, with $\eta(\mcM)$ defined as in Theorem~\ref{th:bound-gammab}.
 \end{theorem}
The following proposition provides sufficient conditions for Assumption~\ref{ass:rational-probabilities} to hold for any extreme points of $\U$.
\begin{proposition}\label{prop:rational-sa-ell-1-ell-infty}
    Assume that for each $(s,a) \in \X \times \A$, $\U_{sa}$ is constructed as in~\eqref{eq:sa-rec-uncertainty-set-l-p}, with $\bm{P}^{0}$ satisfying Assumption \ref{ass:rational-probabilities} and $\alpha_{sa} = \beta_{sa}/m$ for some $\beta_{sa} \in \N$. Then for $p=\infty$, Assumption~\ref{ass:rational-probabilities} holds for any extreme points of $\U$, and for $p=1$, Assumption~\ref{ass:rational-probabilities} holds for any extreme points of $\U$ by replacing $m$ with $m'=2m.$
\end{proposition}
Based on Theorem~\ref{th:bound-gammab-sa}, we obtain the first reduction from robust Blackwell optimality to robust discounted optimality. Since discounted RMDPs can be solved with value iteration or policy iteration, 
 we provide the first algorithms to compute a robust Blackwell-optimal policy for RMDPs with sa-rectangular uncertainty, when the uncertainty set is based on the $\ell_{1}$ or the $\ell_{\infty}$ distance. Note that the classical algorithms for computing Blackwell-optimal policies in MDPs do not extend to RMDPs: they are based on the LP formulation of MDPs, and such a formulation is not known for RMDPs~\citep{grand2022convex}.
\section{Conclusion} We introduce the notion of the Blackwell discount factor for MDPs and robust MDPs and we provide an upper bound in all generality. Based on this upper bound, any progress in solving discounted MDPs, one of the most active research directions in RL, can be combined with our results to obtain new algorithms for computing average and Blackwell-optimal policies.
Our work also opens new research avenues for MDPs and RMDPs. In particular, the proof techniques for our bound on $\gammab$ and $\gammabr$, based on the separation of algebraic numbers, are novel and they could be tightened for specific instances or different optimality criteria, such as bias optimality or $n$-discount optimality. The notion of {\em approximate} Blackwell optimality as well as the existence of the robust Blackwell discount factor for other uncertainty sets, e.g., s-rectangular or non-polyhedral sa-rectangular uncertainty sets,  are also interesting directions of research.


\bibliographystyle{icml2023} 
\bibliography{ref}
\appendix
\section{Proof of Section~\ref{sec:bound-blackwell-discount-factor}}\label{app:bound-blackwell}
In this appendix, we provide the proof for Theorem~\ref{th:bound-gammab}.
As noted in Section~\ref{sec:bound-blackwell-discount-factor}, to bound $\gammab$, it is enough to obtain an upper bound on $\gamma(\pi,\pi',s)$ for any $\pi,\pi' \in \Pi$ and $s \in \X$ such that $\gamma \mapsto v_{\gamma,s}^{\pi} - v_{\gamma,s}^{\pi'}$ is not identically equal to $0$, since $\gammab \leq \max_{\pi,\pi' \in \Pi,s \in \X} \gamma(\pi,\pi',s)$.
\paragraph{Step 1.} We start by studying in more detail the properties of the value functions. The following lemma follows directly from Cramer's rule, as explained in Section~\ref{sec:blackwell-optimality}.
\begin{lemma}\label{lem:value-function}
We have
\begin{equation}\label{eq:cramer-rule-value-function}
    v_{\gamma,s}^{\pi} = \frac{\det\left(\bm{M}(\gamma,s,\pi)\right)}{\det\left(\bm{I} - \gamma \bm{P}_{\pi}\right)},
\end{equation}
with
$\bm{M}(\gamma,s,\pi)$ the matrix formed by replacing the $s$-th column of $\bm{I} - \gamma \bm{P}_{\pi}$ by the vector $\bm{r}_{\pi}$.
\end{lemma}

From Lemma~\ref{lem:value-function}, we have
\[v_{\gamma,s}^{\pi} = \frac{n(\gamma,s,\pi)}{d(\gamma,\pi) }\]
for $n(\gamma,s,\pi) = \det\left(\bm{M}(\gamma,s,\pi)\right)$ and $d(\gamma,\pi) = \det\left(\bm{I} - \gamma \bm{P}_{\pi}\right)$. We choose the letter $n$ for {\em nominator} and the letter $d$ for {\em denominator}.

Note that $\gamma \mapsto n(\gamma,s,\pi)$ is a polynomial of degree at most $|\X|-1$, while $\gamma \mapsto d(\gamma,\pi)$ is a polynomial of degree at most $|\X|$. 

We have, by definition, 
\begin{align*}
    v_{\gamma,s}^{\pi} - v_{\gamma,s}^{\pi'}
    & = \frac{n(\gamma,s,\pi)}{d(\gamma,\pi)} - \frac{n(\gamma,s,\pi')}{d(\gamma,\pi')}\\\
    & = \frac{n(\gamma,s,\pi)d(\gamma,\pi') - n(\gamma,s,\pi)d(\gamma,\pi) }{d(\gamma,\pi)d(\gamma,\pi')}
\end{align*}
Therefore, $v_{\gamma,s}^{\pi} - v_{\gamma,s}^{\pi'} = 0$ for $\gamma \in [0,1)$ implies that $\gamma$ is a root of the following polynomial equation in $\gamma$:
\begin{equation}\label{eq:poly-eq-gamma}
p(\gamma)=0,
\end{equation}
for $p$ the polynomial defined as
\begin{equation}\label{eq:definition-P}
p(\gamma)=n(\gamma,s,\pi)d(\gamma,\pi') - n(\gamma,s,\pi')d(\gamma,\pi).
\end{equation} 
\paragraph{Step 2.} We now study the properties of the polynomial $p$. Note that it is straightforward that $p$ is a polynomial of degree $N=2|\X|-1$.
We first study the properties of the polynomial $\gamma \mapsto
 d(\pi,\gamma)$.
We have the following lemma.
\begin{lemma}\label{lem:sign-denominator}
We have
\[d(\gamma,\pi) > 0, \forall \gamma \in [0,1), \forall \; \pi \in \Pi,\]
and $d(1,\pi)=0, \forall \; \pi \in \Pi.$
\end{lemma}
\begin{proof}[Proof of Lemma~\ref{lem:sign-denominator}]
This lemma follows from the relation between the determinant of a matrix and its eigenvalues, through the characteristic polynomial:
\[d(\gamma,\pi) = \det\left(\bm{I} - \gamma \bm{P}_{\pi}\right) =  \prod_{\lambda \in Sp(\bm{P}_{\pi})} \left(1 - \gamma \lambda \right)^{\alpha_{\lambda}},\]
with $\alpha_{\lambda}$ the algebraic multiplicity of the (potentially complex) eigenvalue $\lambda$ in the spectrum $Sp(\bm{P}_{\pi})$ of $\bm{P}_{\pi}$. Since $\bm{P}_{\pi}$ is the transition matrix of a Markov chain, we know that the modulus of any eigenvalue $\lambda$ of $\bm{P}_{\pi}$ is smaller or equal to $1$. This shows that $d(\gamma,\pi) >0, \forall \; \gamma \in [0,1), \forall \; \pi \in \Pi$. To show $d(1,\pi)=0$, we simply note that $1 \in Sp(\bm{P}_{\pi})$ since $\bm{P}_{\pi}$ is the transition matrix of a Markov chain.
\end{proof}
From Lemma~\ref{lem:sign-denominator} and the definition of $p$ as in ~\eqref{eq:definition-P}, it is straightforward that $p(1)=0$.
\begin{lemma}\label{lem:one-root-nominator}
$\gamma=1$ is a root of $p$.
\end{lemma}
We now bound the sum of the absolute values of the coefficients of $p$.
We have the following theorem.
\begin{theorem}\label{th:bound-sum-abs-values-app}
 The polynomial $m^{2|\X|} \cdot p$ has integral coefficients, potentially negative. The sum of the absolute values of the coefficients of $m^{2|\X|} p$ is bounded by \[L=2 \cdot |\X| \cdot r_{\infty} \cdot m^{2|\X|} \cdot 4^{|\X|}.\]
\end{theorem}
Theorem~\ref{th:bound-sum-abs-values-app} is based on the following three propositions. We note $C_{\ell}^{k}$ the binomial coefficient defined as $C_{\ell}^{k} = \ell!/k!(\ell-k)!.$
\begin{proposition}\label{prop:bound-denominator}
For any $\pi \in \Pi$, the function $\gamma \mapsto d(\pi,\gamma)$ is a polynomial of degree $|\X|$. Moreover, $\gamma \mapsto m^{|\X|} \cdot d(\pi,\gamma)$ is a polynomial with integral coefficients (potentially negative), and the absolute value of its coefficient of degree $k$ is bounded by $m^{|\X|} C_{|\X|}^{k}.$

Therefore, the sum of the absolute values of the coefficients of $\gamma \mapsto m^{|\X|} \cdot d(\pi,\gamma)$ is upper bounded by \[
L_{d}=m^{|\X|} \cdot 2^{|\X|}.\]
\end{proposition}
\begin{proposition}\label{prop:bound-nominator}
For any policy $\pi \in \Pi$ and any state $s \in \X$, the function $\gamma \mapsto n(\gamma,s,\pi)$ is a polynomial of degree $|\X|-1.$ Moreover, $\gamma \mapsto m^{|\X|} \cdot n(\gamma,s,\pi)$ is a polynomial with integral coefficients (potentially negative), and the absolute value of its coefficient of degree $k$ is bounded by $m^{|\X|} \cdot |\X| \cdot r_{\infty} \cdot C_{|\X|-1}^{k} \cdot 2$.

Therefore, the sum of the absolute values of the coefficients of $\gamma \mapsto m^{|\X|} \cdot n(\gamma,s,\pi)$ is upper bounded by \[L_{n}=m^{|\X|-1} \cdot |\X| \cdot r_{\infty} \cdot 2^{|\X|}.\]
\end{proposition}
\begin{proposition}\label{prop:sum-coeff-prod-poly}
Let $P = \sum_{i=0}^{n} a_{i}X^i,Q = \sum_{j=0}^{m} b_{j}X^{j}$. Then $PQ = \sum_{k=0}^{n+m} c_{k}X^k$, $c_{k} = \sum_{i,j; i+j = k} a_{i}b_{j}$.
Additionally, suppose that $\sum_{i=0}^{n} |a_{i}| \leq L_{P},\sum_{j=0}^{m} |b_{j}| \leq L_{Q}$. Then 
\[\sum_{k=0}^{n+m} |c_{k}| \leq L_{P}L_{Q}.\]
\end{proposition}

Combining Proposition~\ref{prop:bound-denominator}, Proposition~\ref{prop:bound-nominator} and Proposition~\ref{prop:sum-coeff-prod-poly} with the definition of the polynomial $p$ as in~\eqref{eq:definition-P} yields Theorem~\ref{th:bound-sum-abs-values-app}. 

To conclude Step 2 of our proof, let us prove Proposition~\ref{prop:bound-denominator} and Proposition~\ref{prop:bound-nominator}. Proposition~\ref{prop:sum-coeff-prod-poly} simply follows from the multiplication rule for polynomials.
\begin{proof}[Proof of Proposition~\ref{prop:bound-denominator}]
  By definition,
  \[ d(\gamma,\pi) = \det\left(\bm{I} - \gamma \bm{P}_{\pi}\right)=\sum_{k=0}^{|\X|} a_{k}\left(\gamma\bm{P}_{\pi}\right),
  \]
  where $\bm{M} \mapsto a_{k}\left(\bm{M}\right)$ is the $(|\X|-k)$-th coefficient of the characteristic polynomial of a matrix $\bm{M}$. By definition, $a_{k}(\bm{M})$ is the sum of all the principal minors of size $k$ of $\bm{M}$ (section 0.7.1, \cite{horn2012matrix}). This first shows that $a_{k}\left(\gamma\bm{P}_{\pi}\right) = \gamma ^k a_{k}\left(\bm{P}_{\pi}\right),$ and therefore, that
\[d(\gamma,\pi) = \sum_{k=0}^{n} \gamma^k a_{k}\left(\bm{P}_{\pi}\right).\]
We will show that 
\[ a_{k}(\bm{P}_{\pi}) \leq  C_{|\X|}^{k}, \forall \; k = 1,...,|\X|.\]
 Let $g$ be a principal minor of $\bm{P}_{\pi}$ of size $k$. By definition, $g$ is the determinant of a submatrix $\bm{M}$ of size $k$ of $\bm{P}_{\pi}$, obtained by deleting rows and columns with the same indices: $g = \det(\bm{M})$. For any matrix square $\bm{M}$, we always have $\det(\bm{M}) = \det(\bm{M}^\top).$ Now Hadamard's inequality shows that $\det(\bm{M}^\top) \leq \prod_{i=1}^{k} \| Col_{i}(\bm{M}^{\top})\|_{2}$, with $Col_{i}(\bm{M}^{\top})$ the $i$-th column of $\bm{M}^{\top}$, and therefore we have $\det(\bm{M}^\top) \leq \prod_{i=1}^{k} \| Col_{i}(\bm{M}^{\top})\|_{1}$. Note that the columns of $\bm{M}^\top$ have $\ell_{1}$-norm smaller than $1$, since $\bm{P}_{\pi}$ is a stochastic matrix, and $\bm{M}$ is a submatrix of $\bm{P}_{\pi}$. Therefore, $g \leq 1$. Because there are $C_{n}^{k}$ possible principal minors of size $k$ of $\bm{P}_{\pi}$, we have $a_{k}(\bm{P}_{\pi}) \leq  C_{n}^{k}, \forall \; k = 1,...,n.$

Of course, we may have $a_{k}(\bm{P}_{\pi}) \notin \Z$. However, for any principal minor $g = \det(\bm{M})$ of $\bm{P}_{\pi}$, we have, by definition the determinant,
\[ \det(\bm{M}) = \sum_{\sigma \in \mathfrak{S}_{k}} \varepsilon(\sigma) \prod_{i=1}^{k}M_{\sigma(i)i}\]
where $\varepsilon(\sigma)$ is the signature of the permutation $\sigma$ and $\mathfrak{S}_{k}$ is the symmetric group, i.e., the group of all permutations of $\{1,...,k\}$. This shows, from Assumption~\ref{ass:rational-probabilities}, that $m^{|\X|} \det(\bm{M}) \in \Z$, and therefore that $m^{|\X|} a_k(\bm{P}_{\pi}) \in \Z$ and that $m^{|\X|} a_k(\bm{P}_{\pi}) \leq m^{|\X|} C_{|\X|}^{k}$.
\end{proof}
\begin{proof}[Proof of Proposition~\ref{prop:bound-nominator}]
Using Laplace cofactor expansions (section 0.3.1, \cite{horn2012matrix}), we have that $n(\gamma,s,\pi)$ is equal to
\begin{equation}\label{eq:laplace-expansion}
 \sum_{s' \in \X} (-1)^{s + s'} \cdot r_{s',\pi(s')} \cdot  \det\left(\left(\bm{I} - \gamma \bm{P}_{\pi}\right)_{\X \setminus \{s'\} \times \X \setminus \{ s\}}\right),
\end{equation}
where $\left(\bm{I} - \gamma \bm{P}_{\pi}\right)_{\X \setminus \{s'\} \times \X \setminus \{s\}}$ is the matrix obtained from $\bm{I} - \gamma \bm{P}_{\pi}$ by removing the $s$-th column and the $s'$-th row.

Note that $\gamma \mapsto \det\left(\left(\bm{I} - \gamma \bm{P}_{\pi}\right)_{\X \setminus \{s'\} \times \X \setminus \{ s\}}\right)$ is a polynomial of degree $|\X|-1$ in $\gamma$. Similarly as for the proof of Proposition~\ref{prop:bound-denominator}, $ \gamma \mapsto m^{|\X|} n(\gamma,s,\pi)$ is a polynomial of degree $|\X|-1$ with integral coefficients. 

Let us consider $\bm{I}_{\backslash \{s',s\}}$ the matrix of dimension $(|\X|-1)\times(|\X|-1)$, obtained by removing the $s$-th column and the $s'$-th row from the identity matrix of dimension $|\X|$, and let us call $\bm{E}_{s'}$ the matrix of dimension $(|\X|-1)\times(|\X|-1)$, where all rows are $\bm{0}^{\tr}$, except the $s$-th row, equal to $\bm{e}_{s'}\tr$.
 
Then $\det\left(\left(\bm{I} - \gamma \bm{P}_{\pi}\right)_{\X \setminus \{s'\} \times \X \setminus \{ s\}}\right)$ is equal to 
\[ \det\left(\left(\bm{I} - \gamma \bm{P}_{\pi}\right)_{\X \setminus \{s'\} \times \X \setminus \{ s\}} + \bm{E}_{s'} - \bm{E}_{s'}\right)\]
and therefore is equal to
\[\det\left(\bm{I}_{\backslash \{s',s\}} + \bm{E}_{s'} - \left( \gamma \bm{P}_{\pi}\right)_{\X \setminus \{s'\} \times \X \setminus \{ s\}} - \bm{E}_{s'}\right).\]
We notice that  $\bm{I}_{\backslash \{s',s\}}+\bm{E}_{s'}$ is a matrix whose rows are exactly the rows of the identity matrix of $\R^{|\X|-1}$, up to a certain permutation $\sigma \in \mathfrak{S}_{|\X|-1}$. Let $\bm{P}^{\sigma} \in \R^{(|\X|-1) \times (|\X|-1)}$ the permutation matrix defined as $P_{ij} = 1$ if $\sigma(j)=i$ and $0$ otherwise. Then 
 for any matrix $\bm{M}$, we have $\det(\bm{P}^{\sigma}\bm{M}) = \det(\bm{P}^{\sigma})\det(\bm{M})=\varepsilon(\sigma) \det(\bm{M}),$ with $\varepsilon(\sigma)$ the signature of the permutation $\sigma$. Since we always have $\varepsilon(\sigma) \in \{-1,1\}$, this shows that 
$\det\left(\left(\bm{I} - \gamma \bm{P}_{\pi}\right)_{\X \setminus \{s'\} \times \X \setminus \{ s\}}\right) $ is equal to
\[ \varepsilon(\sigma)\det\left(\bm{I} - \left(\left( \gamma \bm{P}_{\pi}\right)_{\X \setminus \{s'\} \times \X \setminus \{ s\}} + \bm{E}_{s'}\right)\right).\]
The map $\gamma \mapsto \det\left(\bm{I} - \left(\left( \gamma \bm{P}_{\pi}\right)_{\X \setminus \{s'\} \times \X \setminus \{ s\}} + \bm{E}_{s'}\right)\right)$ is equal to 
\[\sum_{k=0}^{|\X|-1} a_{k}\left(\left( \gamma \bm{P}_{\pi}\right)_{\X \setminus \{s'\} \times \X \setminus \{ s\}} - \bm{E}_{s'}\right)\]
where similarly as for the proof of Proposition~\ref{prop:bound-denominator}, $a_{k}(\bm{M})$ is the $k$-th coefficient of the characteristic polynomial of a matrix $\bm{M},$ i.e., $a_{k}(\bm{M})$ is equal to the sum of all the principal minors of $\bm{M}$ of dimension $k \times k$. Let 
\[ \bm{M} = \left( \gamma \bm{P}_{\pi}\right)_{\X \setminus \{s'\} \times \X \setminus \{ s\}} - \bm{E}_{s'}.\]
Note that $\left(\bm{P}_{\pi}\right)_{\X \setminus \{s'\} \times \X \setminus \{ s\}}$ is a substochastic matrix, i.e., it has non-negative entries and the sum of the entries of each row is smaller or equal to $1$. Note $\bm{M}$ differs from $\left( \gamma \bm{P}_{\pi}\right)_{\X \setminus \{s'\} \times \X \setminus \{ s\}}$ only at the coefficient of index $(s,s')$. Using Hadamard's inequality, we find that 
that
\begin{equation}\label{eq:bound-ak-nomintor}
 a_{k}(\bm{M}) \leq 2 \cdot C_{|\X|-1}^{k}, m^{|\X|} a_{k}(\bm{M}) \in \N.   
\end{equation}
We conclude by combining Equation~\eqref{eq:bound-ak-nomintor} with Equation~\eqref{eq:laplace-expansion}.
\end{proof}
\paragraph{Step 3.}
 We now lower bound the distance between any two roots of $p$ by a scalar $\eta>0$. Since we know that for $\gamma(\pi,\pi',s) \in [0,1)$ and $1$ are two roots of $P$, this will show that
 $\gamma(\pi,\pi',s) < 1 - \eta$. 

Our proof is based on the following theorem.
\begin{theorem}[\cite{rump1979polynomial}]\label{th:separation-algebraic-number-app}
Let $p$ be a polynomial of degree $N$ with integer coefficients, possibly with multiple roots. Let $L$ be the sum of the absolute values of its coefficients. Then the distance between any two distinct roots of $p$ is strictly larger
\[\frac{1}{2N^{N/2+2}\left(L+1\right)^{N}}.\]
\end{theorem}
Recall that both $\gamma(\pi,\pi',s) \in [0,1)$ and $1$ are roots of the polynomial $p$. Therefore, we can combine Theorem~\ref{th:separation-algebraic-number-app} with Theorem~\ref{th:bound-sum-abs-values-app}  to obtain $\gamma(\pi,\pi',s) < 1 - \eta(\mcM)$, with
\[ \eta(\mcM) = \frac{1}{2N^{N/2+2}\left(L+1\right)^{N}}\]
with
\begin{align*}
N & = 2 |\X|-1,\\
L & = 2 \cdot |\X| \cdot r_{\infty} \cdot m^{2|\X|} \cdot 4^{|\X|}.
\end{align*}
This concludes the proof of Theorem~\ref{th:bound-gammab}.
\begin{remark}
Note that \cite{akian2019operator} use Theorem~\ref{th:separation-algebraic-number-app} to obtain a lower bound on the average rewards of any two different policies, in the setting of two-player stochastic games.
\end{remark}
\section{Proof of Section~\ref{sec:robust-mdps}}\label{app:proof-rmdps}
\begin{proof}[Proof of the existence of $\gammabr$.]
Let \[\bar{\gamma}_{{\sf r}}= \max_{\pi,\pi' \in \Pi,s \in \X} \max_{ \bm{P},\bm{P}' \in \U_{\sf ext}}  \gamma(\pi,\pi',s,\bm{P},\bm{P}'),\]
where $\gamma(\pi,\pi',s,\bm{P},\bm{P}')$ is the largest zero of the function
$\gamma \mapsto v^{\pi,\bm{P}}_{\gamma,s} - v^{\pi',\bm{P}'}_{\gamma,s}$ if it is not identically equal to zero, or $\gamma(\pi,\pi',s,\bm{P},\bm{P}')=0$ otherwise. Recall that $\U_{\sf ext}$ is the (finite) set of extreme points of $\U$.
We will show that $\Pi\opt_{\gamma,{\sf r}} = \Pibr,\forall \; \gamma > \bar{\gamma}_{{\sf r}}$. Let $\pi$ be a robust discounted optimal policy for some $\gamma > \bar{\gamma}_{{\sf r}}$. We will prove that $\pi$ is a Blackwell-optimal policy. Since $\pi$ is robust $\gamma$-discounted optimal, we have
\[v^{\pi,\mcU}_{\gamma,s} \geq v^{\pi',\mcU}_{\gamma,s}, \forall 
\; \pi' \in \Pi,\forall \; s \in \X.\]
By definition $v_{\gamma,s}^{\pi,\U} = \min_{\bm{P} \in \U} v_{\gamma,s}^{\pi,\bm{P}}, \forall \; s \in \X.$ From~\cite{iyengar2005robust}, we know that the $\arg \min$ in $\min_{\bm{P} \in \U} v_{\gamma,s}^{\pi,\bm{P}}$ is attained at an extreme point of $\U$.
Therefore, by definition of $\bar{\gamma}_{{\sf r}}$, the function $\gamma \mapsto v^{\pi,\mcU}_{\gamma,s} - v^{\pi',\mcU}_{\gamma,s}$ cannot be equal to $0$ on $(\bar{\gamma}_{{\sf r}},1)$, and therefore it does not change sign, since it is a continuous function. This shows that for all $\gamma > \bar{\gamma}_{{\sf r}}$, we have \[v^{\pi,\mcU}_{\gamma,s} \geq v^{\pi',\mcU}_{\gamma,s}, \forall 
\; \pi' \in \Pi,\forall \; s \in \X.\] 
This shows the existence of the robust Blackwell discount factor $\gammabr$ and that $\gammabr < \bar{\gamma}_{{\sf r}}$.
\end{proof}
\begin{proof}[Proof of Proposition~\ref{prop:rational-sa-ell-1-ell-infty}]
In the proof of this proposition, we use the fact that the worst-case kernel $\bm{P}\opt$ of a policy $\pi$ can be chosen as the $\arg \min$ of the optimization problem $\min_{\bm{p} \in \U_{s\pi(s)}} \bm{p}\tr\bm{v}^{\pi,\U}_{\gamma}$, where $\bm{v}^{\pi,\U}_{\gamma}$ is the worst-case value function of $\pi$. In particular, let $\bm{v} \in \R^{\X}$.

\paragraph{The case $p=\infty$.} 
In this case, there exists a sorting solution to $\min_{\bm{p} \in \U_{sa}} \bm{p}\tr\bm{v}$ for any $(s,a) \in \X \times \A$ and any $\bm{v} \in \R^{\X}$, by sorting $\bm{v}$, see for instance proposition 3 in~\cite{goh2018data}, equation~(9) in~\cite{givan1997bounded}, or appendix~C in~\cite{behzadian2021fast}. In particular, let $(s,a) \in \X \times \A$ and define $\sigma$ the permutation of $\X$ such that $v_{\sigma(1)} \leq ... \leq v_{\sigma(|\X|)}$, and define $i$ as the smaller integer in $\{1,...,|\X|\}$ such that 
\[ \sum_{s'=1}^{i} \left(P^{0}_{sa\sigma(s')} + \alpha_{sa} \right)+\sum_{s'=i+1}^{|\X|} \left(P^{0}_{sa\sigma(s')} - \alpha_{sa}\right) \geq 1.\]
Then a solution to $\min_{\bm{p} \in \U_{sa}} \bm{p}\tr\bm{v}$ is
$p_{\sigma(s')} = P^{0}_{sa\sigma(s')} + \alpha_{sa}$ if $s' < i$, $p_{\sigma(s')} = P^{0}_{sa\sigma(s')} - \alpha_{sa}$ if $s' > i$, and 
\[p_{\sigma(i)} = 1 - \sum_{s' \in \X \setminus \{ i \}} p_{\sigma(s')}.\]
This closed-form shows that for any vector $\bm{v} \in \R^{\X}$, a solution of $\min_{\bm{p} \in \U_{sa}} \bm{p}\tr\bm{v}$ can be found 
as a vector with rational entries with a denominator of at most $m$.

\paragraph{The case $p=1$.} 
In this case, one can show that the optimization problem $\min_{\bm{p} \in \U_{sa}} \bm{p}\tr\bm{v}$ can be formulated as a linear program. Therefore, there exists an optimal basic feasible solution $\bm{p}$ which has the following form by lemma~5.4 and lemma~5.5 in~\cite{ho2021partial}. There exist $j_1, j_2 \in  \mathcal{S}$ such that $j_1 \neq  j_2$ and for each $i\in \mathcal{I} = \mathcal{S} \setminus \{j_1, j_2\}$:
\begin{align*}
  p_i = 0 &\quad \text{or} \quad p_i = P^{0}_{sai}  \\
 p_{j_1} \ge P^{0}_{saj_{1}}  &\quad \text{and}\quad p_{j_2} \le P^{0}_{saj_{2}} \,.  
\end{align*}
  Then, in order for $\bm{p}\in \mathcal{U}_{sa}$ we need the following equalities to hold
\begin{align*}
  p_{j_1} + p_{j_2}  &= 1 - \sum_{i\in \mathcal{I}} p_i  \\
  (p_{j_1} -  P^{0}_{saj_{1}}) + ( P^{0}_{saj_{2}} - p_{j_2}) &= \alpha_{sa} - \sum_{i\in \mathcal{I}} |p_i - P^{0}_{sai} | \,.
\end{align*}
Combining the equalities above yields that
\begin{align*}
  2 p_{j_1} &= \alpha_{sa} - \sum_{i\in \mathcal{I}} |p_i - P^{0}_{sai} |   + P^{0}_{saj_{1}} - P^{0}_{saj_{2}} \\
  &\quad  + 1 - \sum_{i\in \mathcal{I}} p_i \,.
\end{align*}
Because the right-hand side of the equation above is a sum of rational numbers with a denominator of at most $m$, $p_{j_1}$ is also rational with a denominator at most $2m$. Using an analogous argument for $p_{j_2}$, we get that there exists an optimal solution that is rational with a denominator of at most $2m$. 
\end{proof}
\end{document}